\documentclass[sigconf]{aamas} 

\usepackage{balance} 

\setcopyright{ifaamas}
\acmConference[AAMAS '24]{Proc.\@ of the 23rd International Conference
on Autonomous Agents and Multiagent Systems (AAMAS 2024)}{May 6 -- 10, 2024}
{Auckland, New Zealand}{N.~Alechina, V.~Dignum, M.~Dastani, J.S.~Sichman (eds.)}
\copyrightyear{2024}
\acmYear{2024}
\acmDOI{}
\acmPrice{}
\acmISBN{}

\usepackage{amsmath}
\usepackage{mathtools}
\usepackage{amsthm}
\newtheorem{thm}{Theorem}

\usepackage{subcaption}
\def\argmax{\operatornamewithlimits{argmax}}

\newtheorem{assumption}{Assumption}

\newcommand{\citeay}[1]{\citeauthor{#1} (\citeyear{#1})}

\acmSubmissionID{747}

\title[PBIM]{Potential-Based Reward Shaping For Intrinsic Motivation}

\author{Grant C. Forbes, Nitish Gupta, Leonardo Villalobos-Arias, \\ Colin M. Potts, Arnav Jhala, David L. Roberts}
\affiliation{
  \institution{North Carolina State University}
  \city{Raleigh}
  \country{United States}}
\email{{gforbes,nagupta,lvillal,cmpotts,ahjhala,dlrober4}@ncsu.edu}

\begin{abstract}
Recently there has been a proliferation of intrinsic motivation (IM) reward-shaping methods to learn in complex and sparse-reward environments. These methods can often inadvertently change the set of optimal policies in an environment, leading to suboptimal behavior. Previous work on mitigating the risks of reward shaping, particularly through potential-based reward shaping (PBRS), has not been applicable to many IM methods, as they are often complex, trainable functions themselves, and therefore dependent on a wider set of variables than the traditional reward functions that PBRS was developed for. We present an extension to PBRS that we prove preserves the set of optimal policies under a more general set of functions than has been previously proven. We also present {\em Potential-Based Intrinsic Motivation} (PBIM), a method for converting IM rewards into a potential-based form that is useable without altering the set of optimal policies. Testing in the MiniGrid DoorKey and Cliff Walking environments, we demonstrate that PBIM successfully prevents the agent from converging to a suboptimal policy and can speed up training.
\end{abstract}


\newcommand{\BibTeX}{\rm B\kern-.05em{\sc i\kern-.025em b}\kern-.08em\TeX}

\begin{document}


\pagestyle{fancy}
\fancyhead{}


\maketitle 


\section{Introduction}

An increasing amount of work in reinforcement learning (RL) uses intrinsic reward functions, in addition to environmental rewards, to speed convergence to reasonable policies. This approach is particularly widespread in sparse-reward problems, or those that are exploration-heavy, and has had much success in these domains \cite{bellemare2016unifying, pathak2017curiosity, burda2018exploration}.

However, adding a secondary reward term may lead to changes in the set of optimal policies, with unintended, and potentially adverse, consequences. For example, \citet{burda2018large} show that an intrinsic reward that incentivizes visiting areas of the state space where the agent is less able to predict what will happen can result in the agent becoming ``addicted'' to watching a screen with flashing random images. \citet{amodei2016concrete} further discuss related issues. 

These issues can be mitigated through hyperparameter tuning---i.e., by multiplying the intrinsic rewards by some $\alpha$ and decreasing $\alpha$ until the problematic behavior disappears. \citet{chen2022redeeming} implemented an automated generalization of this method. However, this approach may decrease the utility of intrinsic motivation and is not generally guaranteed to preserve the optimal policy set. As an alternative, we extend the potential-based, policy-preserving reward shaping term of \citet{ng1999policy} to arbitrary reward functions, and show that this preserves the set of optimal policies. We contribute:

1. \textbf{An extension of potential based reward shaping (PBRS) to potential functions of arbitrary variables in episodic environments, and a proof that this extension does not alter the set of optimal policies}. We derive an accompanying boundary condition that serves as a sufficient condition for preserving optimality, and this allows for extending PBRS to reward functions, like intrinsic motivation (IM), which are dependent on a more general set of variables than has been addressed previously. 

2. \textbf{A novel method for converting any arbitrary reward function into this extended potential form}, maintaining the benefits of that function while mitigating its drawbacks, by guaranteeing that such a shaping reward will not alter the set of optimal policies in the underlying environment, and thus cannot be ``hacked'' by an agent that has converged to an optimal policy.

3. \textbf{An empirical demonstration that our method is effective} as both a safety measure to prevent hacking of intrinsic rewards and, in certain cases, to speed up training.

This paper expands on work originally presented in \cite{aaai24}.
\section{Related Work}

The two fields of study most directly relevant to our work are reward shaping---potential-based reward shaping, in particular---and intrinsic motivation. We will review the former of these first, focusing particularly on the theoretical extensions that have been made, and then the latter.

\subsection{Potential-Based Reward Shaping}

Reward shaping is the practice of adding some additional reward to an environment, usually with the goal of accelerating training. We define a Markov Decision Process (MDP) as a tuple $M = (S, A, T, \gamma , R)$. Here $S$ is the state space, and $A$ is the action space. $T : S \times A \times S \to [0,1]$ is the state transition function so that $T(s, a, s') = P(s_{t+1} = s' | s_t = s, a_t = a) $ is the probability of transitioning to state $s'$ by taking action $a$ in state $s$ (throughout this paper, ``$'$'' will be used to notate a subsequent time step). $\gamma \in [0,1]$ is the discount factor, and $R : S \times A \times S \to \Re $ is the reward function. Reward shaping for $M$ is training the agent in some second MDP $M' = (S, A, T, \gamma, R')$, where $R' = R + F$, and $F$ is the shaping reward (this problem and notation can be extended to partially observable MDPs (POMDPs) as well). It is easy to accidentally choose $F$ in a way that leads to unexpected behavior (see \cite{randlov1998learning}), because the set of optimal policies for $M'$ is not guaranteed to be the same as for $M$. \citet{ng1999policy} showed that the set of optimal policies for an MDP is unchanged by the addition of a shaping reward of the form 
\begin{equation}\label{Ng_potential}
    F(s,s') = \gamma \Phi(s') - \Phi(s).
\end{equation}
If a shaping reward of this form was added, then the Q-function of the new MDP, $M'$, would be equal to the old Q-function, plus a potential with no action-dependence:
\begin{equation}
     Q^*_{M'}(s,a) = Q^*_M(s,a) - \Phi(s).
\end{equation}
Because there is no action-dependence, choosing an action to maximize its sum will also maximize the original Q-function. This was only proven for MDPs that are infinite-horizon or contain a single ``absorbing state.'' 

\citet{wiewiora2003principled} extended this treatment to reward shaping terms of the form
\begin{equation} \label{wiewora_potential}
F(s,a,s',a') = \gamma \Phi(s', a') - \Phi(s,a),
\end{equation} 
allowing for action-depended ``advice'' potentials: they note, however, that in order for the theoretical guarantees of \citet{ng1999policy} to hold, the potential $\Phi(s,a)$ must be added back to the Q-value during policy training. 
Also note the dependence of Equation~\ref{wiewora_potential} on $a'$, which is necessary for this formulation to work, but requires a reward of the form $R(s,a,s',a')$, rather than the more standard $R(s,a,s')$.\footnote{\citet{wiewiora2003principled} also include a ``lookback'' formulation, wherein knowledge of $a'$ is not required, but rather $a_{t-1}$. This formulation, while not dependent on $a'$, still takes two action values as arguments, and is thus still nonstandard.} \citet{wiewiora2003potential} showed that this form of reward shaping parameter is theoretically equivalent to a thoughtful initialization of parameters for the policy parameterization. 

\citet{devlin2012dynamic} extend this formulation further, showing that any shaping reward of the form 
\begin{equation} \label{devlin_potential}
    F(s,t,s', t') = \gamma \Phi(s',t') - \Phi(s,t),
\end{equation}
where $t$ and $t'$ refer to the corresponding time steps of $s$ and $s'$, would not alter the set of optimal policies in a single-agent problem, or the Nash equilibrium in a multi-agent problem. They also prove that, unlike Equation~\ref{wiewora_potential}, shaping rewards of this form cannot be reduced to parameter initialization in the same way shown by \citet{wiewiora2003potential}.

\citet{harutyunyan2015expressing} combine the expansions to \citet{ng1999policy} of \citet{wiewiora2003principled} and \citet{devlin2012dynamic} into a general form
\begin{equation} \label{harutyanyun_potential}
    F(s,a,t,s',a',t') = \gamma \Phi(s',a',t')-\Phi(s,a,t)
\end{equation}
and demonstrate a method for converting shaping functions of the form $R(s,a,s')$ to this potential-based formulation\footnote{\citet{harutyunyan2015expressing} describe these as ``arbitrary reward functions'' that are the most general form of reward function that exists in a traditionally-defined MDP. This is required by the way the proof assumes convergence to a TD fixed point. This assumption is not guaranteed for rewards that are arbitrary in the wider sense we use, including most forms of IM---these are often not stable across time and thus a time-independent value function cannot be expected to converge.}. However, \citet{behboudian2022policy} prove that this formulation does not actually maintain an optimal policy in a general sense. They present PIES, an alternative method to preserve optimality, by iteratively decreasing the coefficient of the shaping rewards until it reaches zero. While this alternative method does indeed preserve optimality, it does so effectively through removing $F$ entirely from a portion of the training, and thus requires a balance between utilizing the benefits of the shaping rewards and mitigating their drawbacks.\footnote{Compare to \citet{chen2022redeeming}, who tune shaping reward coefficients for IM specifically, but not to zero.}
All works cited thus far have focused on the initial domain of \citet{ng1999policy}, which is environments that are either infinite-horizon, or which have a set absorbing state that terminates the episode. This distinction is formally emphasized by \citet{eck2016potential}, who also make a further extension of PBRS to Partially-Observable Markov Decision Processes (POMDPs). Noting this shortcoming, \citet{grzes2017reward} extended PBRS to episodic environments, which terminate after some final time step $N$. They note that, in such an environment, when adding the shaping reward in Equation~\ref{Ng_potential}, there is an additional term of difference between the episodic returns of $M'$ and $M$:
\begin{equation}\label{Grzes_contribution}
    U_{M'}(\tau_N) = U_M(\tau_N) + \gamma^N\Phi(s_N) - \Phi(s_0),
\end{equation}
where $U_{M}(\tau_N)$ is the cumulative return of an agent on MDP $M$ under the state-action trajectory $\tau_N$. The latter of these terms, present even in the infinite-horizon case \cite{ng1999policy}, has no action dependence, and thus cannot affect the set of optimal policies. The former, however, is implicitly action-dependent through the agent's ability to affect its final state, and thus poses a problem for maintaining the optimality of the learned policy. \citet{grzes2017reward} addresses this problem by adjusting the potential added at the end of the episode. The potential function essentially then becomes time-dependent, as defined in Equation~\ref{Ng_potential} at all time steps except the last of an episode, where it is zero. Formally:
\begin{equation} \label{grzes_potential}
    \Phi_n = \begin{cases}  0 &  \text{if } n = N \\  \Phi(s) &  \text{otherwise}.\end{cases}
\end{equation}
This treatment applies not only to fixed values of $N$, but to environments where a number of different states could be terminal states, and termination could happen at differing times: here $N$ is treated as ``the time at which the episode ends,'' and can freely vary from episode to episode. This truncating of the potential in the last time step ensures that the problematic term from Equation \ref{Grzes_contribution} will always equal zero, and thus restores the desired optimality guarantees.

\citet{goyal2019using} extended PBRS to a potential based on an ``action frequency vector,'' which contains information about the agent's trajectory over some slice of time, while preserving optimality.

We are extending the potential-based formulation further to accommodate potentials that are a function of an arbitrary set of variables. Most commonly, this will simply be $\Phi(\tau_{0}, \tau_{1}, ... \tau_{M})$, where $\Phi$ is the shaping potential, $M$ is the total number of episodes during training, and $\tau_m = (s_0, a_0, s_1, a_1, ... a_{N_m-1}, s_{N_m})$ is  the full trajectory of states and actions during episode $m$ of training. This is sufficiently general to accommodate most IM terms, such as ICM \cite{pathak2017curiosity} or RND \cite{burda2018exploration}. However, to emphasize a generality that could in principle extend beyond this, rather than writing the dependence explicitly (as in, i.e., the $\Phi(s,a,s',a')$ of \citet{wiewiora2003principled}), we will write either $F_n$ or $\Phi_n$ for simplicity of notation and to emphasize that we are dealing with arbitrary variable-dependence.

\subsection{Intrinsic Motivation} \label{IM}

Intrinsic motivation has proven increasingly useful for complex or sparse-rewards environments in recent years. However, the actual reward shaping terms used in the IM literature lie almost universally outside of the traditional MDP and POMDP frameworks, as they cannot be written as a function of a single state transition, $R(s,a,s')$.

A large portion of IM literature is focused on incentivizing exploration, particularly in sparse-reward environments. Simple versions exist, such as incentivizing taking actions that have not been taken recently \cite{sutton1990integrated} or keeping a tabular list of how often each state has been explored, and rewarding less-visited ones \cite{strehl2008analysis}. Recently, more complex exploration rewards have been developed. Tabular methods have been extended to larger, more complex environments through ``pseudo-counts'' \cite{bellemare2016unifying,ostrovski2017count}, which use a learned representation of the (potentially continuous) state space. Curiosity-based methods like Intrinsic Curiosity Module (ICM) \cite{pathak2017curiosity} reward agents for ``surprising'' (maximizing the error rate of) an auxiliary network trained to predict the environment state dynamics. Random Network Distillation (RND) \cite{burda2018exploration}, similarly, rewards agents for fooling a predictor in a random feature space.
Another common IM method relies on ``empowerment'' \cite{mohamed2015variational}, which is a mutual information metric between the agent's actions and future states. \citet{raileanu2020ride} aims to maximize the impact of an agent's actions on a learned state representation.

All examples thus far have used IM as a method for supplementing (usually sparse) base extrinsic rewards. Recently, IM without the base reward, either to learn skills to be applied later \cite{eysenbach2018diversity} or to replace external rewards entirely \cite{burda2018large}, has gained attention.


There has been some prior work on the potential risks of IM. Examples include the ``noisy TV'' problem, where an agent with an exploration term advising it to seek novelty can get distracted from a base task by some particularly stochastic object in its environment \cite{burda2018large}. There is also a tendency of other exploration terms less susceptible to the noisy TV problem, such as RND, still causing agents to become noticeably ``risk seeking'' once they've exhausted all easy-to-obtain intrinsic rewards \cite{burda2018exploration}. There is a large body of theoretical work in this area \cite{amodei2016concrete}, but empirical study remains sparse. We hope that our method can serve as a tool to better assist empirical research in this area.

There has been some other work in the area of mitigating potential adverse effects of IM terms, coming mostly in the form of hyperparameter tuning. \citet{chen2022redeeming}, for example, utilize a clever method of automatically tuning up exploration coefficients in exploration-heavy environments and tuning them down in regions where IM is less beneficial. Our solution differs from this in two key ways. Firstly, and most importantly, it delivers vital theoretical guarantees that the set of optimal policies will remain unchanged, and thus that any convergence guarantees apply within the new MDP. Secondly, while \cite{chen2022redeeming} requires additional hyperparameters, network architecture, and optimization steps beyond that for the combined loss function, our method requires virtually no additional computational overhead, and addresses the problem solely by adjusting the reward shaping term to one that guarantees an unchanged set of optimal policies. 
\section{Main Results}

Here we demonstrate an extension of PBRS to arbitrary potential functions that satisfy a boundary condition. Motivated by this condition, we then develop a method for converting almost (see Assumption~\ref{not_empowerment}) any arbitrary reward function to a potential that preserves optimality. We developed two versions of this conversion method, one normalized and one non-normalized. We derive the initial boundary condition in Section~\ref{extending_sec}, and discuss it in relation to prior optimality-preserving PBRS results. We develop the resulting reward-converting methods in Section~\ref{converting}.

\subsection{Extending Potential-Based Reward Shaping to Functions of Arbitrary Variables} \label{extending_sec}

In an episodic environment, we normally want to choose a policy $\pi$ so as to optimize the value function
\begin{equation} \label{V}
    V_{M}^\pi = \displaystyle \mathop{\mathbb{E}}_{a \sim \pi, s \sim T, R_n \sim R} U_M^\pi.
\end{equation}
Here $U_M^\pi$ is the cumulative discounted return:
\begin{equation} \label{U}
     U_M^\pi = \sum_{n = 0}^{N-1} \gamma^n R_n,
\end{equation}
where the rewards $R_n$ are sampled from acting under policy $\pi$ according to the transition dynamics and reward function of environment $M$. Note that we are considering the general case where the reward function itself need not be deterministic. We also want to define the discounted future return at some arbitrary time step $t$: 
\begin{equation} \label{U_t}
     U_{M,t}^\pi = \sum_{n = t}^{N-1} \gamma^{n-t} R_n,
\end{equation}
the expectation of which is $V_{M,t}^\pi$. Given this, an optimal policy under $R$ for environment $M$ at time $t$ will satisfy
\begin{equation}
    \pi_{M}^* = \argmax_\pi( V_{M,t}^\pi).
\end{equation}
This optimal policy $\pi_M^*$ will also satisfy
\begin{equation}\label{Q_argmax}
     \pi_{M}^*(s) = \argmax_{a_t}( Q_{M,t}^*)
\end{equation}
where
\begin{equation}\label{Q}
    Q_{M,t}^\pi = R_t + V_{M,t+1}^\pi,
\end{equation}
and $Q_{M}^*$ is taken to be $Q_{M}^\pi$ of the optimal policy $\pi = \pi^*_M$. 
If we now define a new environment $M'$ equivalent to $M$ but with the addition of a shaping reward
\begin{equation} \label{shaping_reward}
    F_n = \gamma \Phi_{n+1} - \Phi_n,
\end{equation}
then we can calculate the return for a trajectory in $M'$
\begin{align} \label{U_M_prime}
    {U_{M',t}^\pi} &{= \sum_{n=t}^{N-1} \gamma^{n-t} (R_n + F_n)} \\
    \label{U_M_prime2}
         &{= \sum_{n = t}^{N-1} \gamma^{n-t} (R_n + \gamma \Phi_{n+1} - \Phi_n).}
\end{align}
In order to prove that adding a shaping reward of the form in Equation \ref{shaping_reward} does not alter the set of optimal policies of the underlying environment, it is sufficient to prove that, at every state and timestep, choosing $a$ to optimize $Q_{M',t}^*$ will necessarily optimize $Q_{M,t}^*$ as well, and vice versa. To investigate the conditions under which this relation will hold, we can reduce Equation~\ref{U_M_prime2} to
\begin{align} \label{new_potential}
{U_{M',t}^\pi =} &{\sum_{n=t}^{N-1} \gamma^{n-t}R_n + \sum_{n=t}^{N-1} \gamma^{n-t}(\gamma\Phi_{n+1} - \Phi_n) }\\
  {=}& { U_{M,t}^\pi + \gamma \Phi_{t+1} - \Phi_t + \gamma^2 \Phi_{t+2} - \gamma \Phi_{t+1}  + }\\
  &{\gamma^3\Phi_{t+3} - \gamma^2 \Phi_{t+2} + \cdots + \gamma^{N-t}\Phi_N - \gamma^{N-(t+1)}\Phi_{N-1} }\\
  {=}&{ U_{M,t}^\pi + \gamma^{N-t}\Phi_N - \Phi_t.}
\end{align}
This is essentially the derivation for Equation~\ref{Grzes_contribution} by \citeay{grzes2017reward} with a potentially non-Markovian $\Phi_t$, and generalized to apply to all time steps, rather than just $t=0$. Through an application of Equation \ref{V} and Equation \ref{Q}, this becomes
\begin{equation}
    {
    Q_{M',t}^\pi = Q_{M,t}^\pi + \mathop{\mathbb{E}}_{a \sim \pi, s \sim T, R_n \sim R} \left( \gamma^{N-t}\Phi_N - \Phi_t \right).
    }
\end{equation}
Here we see that $Q_{M',t}^\pi$ differs in expectation from $Q_{M,t}^\pi$ by two terms. If the sum of these terms is constant with respect to $a_t$, then Equation \ref{Q_argmax} can be applied to show the equivalence of optimal policies between these two environments. This leads us to the condition
\begin{equation}\label{boundary_condition}
     \mathop{\mathbb{E}}_{a \sim \pi, s \sim T, R_n \sim R} \left( \gamma^{N-t}\Phi_N - \Phi_t \right) = \Phi'_t \qquad \forall t \in (0,1,...N-1),
\end{equation}
where $\Phi'_t$ is some arbitrary function that is constant with respect to action $a_{t}$. From here, we can state Theorem \ref{extending_thm}:
\begin{thm}[Sufficient Condition For Optimality]\label{extending_thm}
    The addition of a shaping reward $F_t = \gamma \Phi_{t+1} - \Phi_t$ leaves the set of optimal policies unchanged if Equation \ref{boundary_condition} holds.
\end{thm}
\begin{proof}
Given Equation \ref{boundary_condition}, then $\forall t \in (0,1,...,N-1)$,
\begin{align} \label{policy_set_equivalence}
    {\pi_{M'}^*(s)} &{= \argmax_{a_t}( Q_{M',t}^*) }\\
    &{= \argmax_{a_t} (Q_{M,t}^* + \displaystyle \mathop{\mathbb{E}} \left( \gamma^{N-t}\Phi_N - \Phi_t \right)) }\\
    &{= \argmax_{a_t} (Q_{M,t}^* + \Phi'_t)}\label{prephi} \\
                    &{= \argmax_{a_t} (Q_{M,t}^*) = \pi_M^*(s).}\label{postphi}
\end{align}
Note the step between Equations \ref{prephi} and \ref{postphi}: here we are relying on the $a_t$-independence of $\Phi'_t$ to ensure it doesn't affect the $\argmax_{a_t}$ term.\footnote{This step is similar to a step in the central proof of \cite{ng1999policy}.} This is equivalent to stating the set of optimal policies is unchanged by the shaping reward. 
\end{proof}
It is worthwhile to briefly examine what prior work has done to preserve the condition in Equation \ref{boundary_condition}, in order to emphasize that this is the most general treatment of this problem to date, and to situate it within prior literature. In a non-episodic setting, the $\gamma^{N-t}\Phi_N$ term either drops out (in the infinite-horizon setting) or is definitionally independent of $a_t$ (in the setting with a set absorbing state). Thus, much prior work in this area has focused on solely the $-\Phi_t$ term. This has been dealt with by either restricting the potential to be independent of $a_t$ \cite{ng1999policy, devlin2012dynamic}, restricting it to be independent of $a_t$ in the limit as training continues \cite{behboudian2022policy}, or subtracting this potential where appropriate to accommodate its $a_t$-dependence \cite{wiewiora2003principled}. All of the restrictions to $\Phi$ introduced by these methods can be viewed as subsets of the general class of shaping functions that satisfy Equation \ref{boundary_condition}. Similarly, prior work in episodic PBRS has restricted itself to $-\Phi_t$ terms that are independent of $a_t$, and thus has dealt with the $\gamma^{N-t}\Phi_N$ term by setting $\Phi_N = 0$ \cite{grzes2017reward}. Again, these solutions combined, while a valid subset of the larger solution space for Equation \ref{boundary_condition}, excludes an important set of solutions in which each of these terms, while individually $a_t$-dependent, have this dependence cancel out when they are summed together. As we will see, these solutions have incredible potential applications for novel shaping functions, particularly as a method to incorporate IM methods without changing the optimal policy of the underlying environment.

\subsection{Converting Functions of Arbitrary Variables to Potential-Based Reward Functions} \label{converting}

All of the IM examples we cited above can change the set of optimal policies, with possibly adverse effects. Thus mitigating these effects, and using IM while guaranteeing the set of optimal policies isn't altered, is highly desirable. We present a practical and straightforward way to convert most IM rewards to a form that is guaranteed not to alter the set of optimal policies. More formally, we present a method that guarantees not to alter the optimal policy for an IM whose terms do not depend on the future actions of the agent. We call this approach Potential Based Intrinsic Motivation~(PBIM).

The trick is to realize that, in all time steps but the last, any arbitrary reward function (including IM) \emph{is already a difference of a potential function in the proper form} due to the recursive relation between rewards and their respective cumulative returns. If we define $F_t$ to be an arbitrary intrinsic reward at time step $t$, and $U_t^\pi$ to be the cumulative discounted intrinsic reward sampled from following policy $\pi$ at time step $t$\footnote{Here, we drop the $M$ subscript for simplicity. Note however that in this section, $U_t^\pi$ exclusively denotes the \emph{intrinsic} discounted return, rather than the sum of intrinsic and extrinsic returns. 
}, then we can rewrite the standard recursive relation between them as
\begin{equation} \label{rewritten_Bellman}
F_t = U_t^\pi - \gamma U_{t+1}^\pi.
\end{equation}

This is conveniently similar to the necessary potential formulation in Equation~\ref{shaping_reward}. In fact, if we choose $\Phi_t = -U_t^\pi$, these equations become identical. 

Choosing this form for $\Phi_t$ may initially appear untenable in a wide variety of environments, as it seems to imply that we will need to know $U_t$ before beginning training. This would presuppose a level of knowledge about the environment and future trajectory of the agent that is unrealistic. However, this is not the case: we don't have to actually know $U_t$ in order to set it equal to $\Phi_t$. We only have to know $F_t$, as it is already a difference of the requisite potential in Equation \ref{rewritten_Bellman}, even if we don't know that potential itself.

If we thus choose $\Phi_t = -U_t^\pi$ and implement an IM term normally, we can investigate under what conditions the optimal policy set is preserved by examining Equation \ref{boundary_condition} under this condition. It becomes
\begin{equation}\label{potential_boundary}
     \mathop{\mathbb{E}}_{a \sim \pi, s \sim T, R_n \sim R} \left( U_t^\pi - \gamma^{N-t}U_N^\pi\right) = \Phi'_t \qquad \forall t \in (0,1,...N-1).
\end{equation}
This condition is not satisfied by default for most IM terms, as $U_t^\pi$ will be action-dependent in most interesting environments. It is also unnecessarily complicated in this formulation, as $U_N=0$ for any environment in an episodic setting. These observations motivate the choice of potential
\begin{equation} \label{initial_potential}
    \Phi_t = \begin{cases} -U_0^\pi/\gamma^N, &  \text{if } t = N \\  -U_t^\pi, &  \text{if }  t \neq N, \end{cases}
\end{equation}
which is similar to choosing $\Phi_t = -U_t^\pi$, but with the crucial exception that $\Phi_N =-\frac{U_0^\pi}{\gamma^N}$, by virtue of $N$ being the last time step in the episode. Our choice of potential here for the $t=N$ case is motivated by setting $\Phi'_t$ in Equation \ref{potential_boundary} to 0 for the $t=0$ case, and solving for $\Phi_N$.

Thus, if we have a shaping reward $F_t$, and we want to utilize some optimality-preserving permutation of it $F'_t$ utilizing the potential of Equation \ref{initial_potential}, we can use
\begin{equation} \label{naive_conversion}
    F'_t = \begin{cases}  \sum_{n = 0}^{N-2} -\gamma^{n+1-N}F_n, &  \text{if } t = N - 1 \\  F_t, &  \text{if } t \neq N - 1,\end{cases}
\end{equation}
which is simply $F'_t = \gamma\Phi_{t+1} - \Phi_t$ with the potential function defined in Equation \ref{initial_potential}. 

Equation \ref{naive_conversion} has an incredibly intuitive and appealing interpretation. It is equivalent to implementing the shaping reward ``normally'' until the very last time step, at which point the total discounted rewards are subtracted in order to ensure Equation \ref{boundary_condition} still holds. Described this way, it is both simple to understand and straightforward to implement.

Equation~\ref{naive_conversion} also has the advantage that it makes it particularly difficult for most agents to ``figure out'' that optimizing intrinsic motivation does nothing to increase their value function in the long run, because the adjustment term is at the very end of a given episode. This extends the reward horizon, to use the terminology of \citet{laud2004theory}, or the time delay between an action and the (intrinsic) returns dependent on that action. This makes it intentionally difficult for the agent to discover that IM doesn't ever actually affect the final return of an episode (because an appropriately discounted quantity will always be deducted later). Much work has gone into the goal of shortening the reward horizon on various problems, oftentimes through reward shaping terms (see, for example, Theorem 3 of \citet{Ngthesis}), but this work shows that actually \emph{increasing} the reward horizon for the futility of pursuing IM can be useful---it allows these rewards to still give hints to the agent, without being immediately discovered as ``worthless'' in the long run. The agent will then seek these rewards in the short term, but discard them in the long term insofar as optimizing for them would deviate from an optimal policy.

For the formal proof that Equation \ref{naive_conversion} leaves an optimal policy unaltered, we must make a single assumption about $F_t$ that limits the scope of rewards our method applies to:
\begin{assumption} \label{not_empowerment}
    $F_t$ is constant with respect to $a_{t'>t} \forall t, t' \in (0,1,...N-1).$
\end{assumption}
This assumption is quite general, and holds for the majority of IM in the literature, including state-of-the-art exploration methods such as ICM and RND. Note that this assumption generally holds for action-dependent IM, so long as that action-dependence does not extend to \emph{future} actions, but is restricted to actions taken by the agent at the current time step and/or prior ones. The key example in the literature for which this assumption does not hold is empowerment \cite{mohamed2015variational}, in which states are given intrinsic weight that is based in part on future actions. Addressing these sorts of IM terms is left to future work; in this paper, we focus on the bulk of IM, for which our method is appropriate. We can now prove Theorem \ref{theorem_2}:
\begin{thm}[PBIM Preserves Optimality]\label{theorem_2}
The addition of a shaping reward $F'_t$ of the form in Equation \ref{naive_conversion} leaves the set of optimal policies unchanged if Assumption \ref{not_empowerment} holds.
\end{thm}
\begin{proof}
    The potential of Equation \ref{naive_conversion} takes the form of Equation \ref{initial_potential}. With this choice of potential, the left side of Equation \ref{boundary_condition} becomes, in expectation,
\begin{align}
    U_t^\pi - \frac{U_0^\pi}{\gamma^t} &{= \sum_{n = t}^{N-1} \gamma^n F_n  - \sum_{n = 0}^{N-1} \gamma^{n-t} F_n }\\
    &{= \left( \sum_{n = t}^{N-1} \gamma^n F_n - \sum_{n = t}^{N-1} \gamma^n F_n\right) - \sum_{n=0}^{t-1} \gamma^{n-t} F_n }\\
    &{= - \sum_{n=0}^{t-1} \gamma^{n-t} F_n.}
\end{align}
This term depends simply on the discounted sum of all the IM rewards up to, but not including, $F_t$. From Assumption \ref{not_empowerment}, this has no $a_t$-dependence, and thus Equation \ref{boundary_condition} holds. From this, Theorem \ref{extending_thm} can be applied to show the optimal policy set is unchanged.
\end{proof} 

While we have just shown it does conserve the optimal policy, this form of PBIM has some potentially undesirable effects in practice. In particular, it may still bias the agent in the short term not only to prefer intrinsic rewards, but also to (temporarily) learn some false relationships between the reward distribution of the state space that then need to be unlearned. Particularly, if $F_t$ is consistently positive (as is the case with most exploration-based IM), then $F'_t$ will also tend to be consistently positive, \emph{except for in the last time step of an episode}, where it will be extremely negative, in order to offset the cumulative positive reward. This may cause an agent to initially learn that areas of the state space it is in towards the end of an episode are ``bad'', and areas of the state space towards the beginning of an episode are ``good.'' While Theorem \ref{theorem_2} ensures these associations will eventually be unlearned, we would prefer to not learn them to begin with, as they may needlessly slow down training: particularly in exploration-focused environments, where they are often precisely the opposite of true. To mitigate this potential issue, we introduce a normalized variation of PBIM, in which $F_t$ is replaced with
\begin{equation} \label{conversion}
F'_t = \begin{cases}  \sum_{n = 0}^{N-2} -\gamma^{n+1-N}F'_n, &  \text{if } t = N -1\\  F_t - \bar{F}, &  \text{if } t \neq N-1, \end{cases}
\end{equation}
where $\bar{F}$ is the expectation value of $F$ across prior training. This modified form ensures that the expected IM for both final and non-final time steps is 0, and thus these undesirable associations will not occur. In practice, $\bar{F}$ is calculated by taking a running mean of the previous intrinsic rewards collected across all workers during a single training epoch. 

The corresponding potential for Equation \ref{conversion} is
\begin{equation} \label{phi_set}
    \Phi_n = \begin{cases} \frac{-U'_0}{\gamma^N}, & \text{if } n = N \\ -U_n -\frac{\bar{F}}{\gamma - 1}, & \text{if } n \neq N, \end{cases}
\end{equation}
where $U'_0$ is the cumulative discounted mean-adjusted intrinsic return.
 The first case of this correspondence follows straightforwardly from the definition of the intrinsic return. The second case gives us back:
\begin{align}
    F'_{n \neq N} &= \gamma \Phi_{n+1} - \Phi_n \\
             &= \gamma ( - U_{n+1} - \frac{\bar{F}}{\gamma - 1}) + U_n + \frac{\bar{F}}{\gamma - 1}\\
             &= F_n - \frac{\gamma \bar{F}}{\gamma - 1} + \frac{\bar{F}}{\gamma - 1} \\
             &= F_n - \frac{\gamma-1}{\gamma-1}\bar{F} \\
             &= F_n - \bar{F}.
\end{align}
Because this $F_{n \neq N}$ term has an expected value of zero, $F_N$ will then similarly have an expected value of zero, as it will simply be the discounted sum of quantities with expectation zero. Additionally, because in either case, Equation \ref{conversion} differs from Equation \ref{naive_conversion} only by the addition of a constant factor of $\bar{F}$, and $\bar{F}$ is never dependent on the action $a_t$, this formulation satisfies Equation \ref{boundary_condition} as well. So we now have two formulations -- Equation \ref{naive_conversion} and Equation \ref{conversion} -- that can be used to implement IM without changing the optimal policy set of the underlying environment.

\section{Empirical Demonstration}\label{empirical}
We empirically demonstrate the efficacy of our method for both an exploration-based tabular IM reward and for Random Network Distillation (RND) \cite{burda2018exploration}. We focused on environments with a knowable set of optimal policies where the base IM demonstrably alters performance and other IM approaches not guaranteeing optimality fail to converge towards an optimal solution. The former of these demonstrations shows our method's potential to speed up convergence when compared to either a baseline IM or no IM, while the latter demonstrates our method's ability to preserve an agent's convergence to an optimal policy, even with an IM that would otherwise explicitly alter the optimality of that policy.
\subsection{MiniGrid DoorKey}\label{minigrid}

We demonstrate an improvement in both speed of convergence and performance of the converged-to policy when using our method to confer optimality guarantees to a tabular exploration reward term in the MiniGrid DoorKey 8x8 environment \cite{minigrid}. This environment
challenges the agent to reach a goal state in the bottom-right corner by picking up a key, carrying it to a door, then unlocking that door. The environment itself has sparse rewards, returning a reward of 1 for successfully reaching the goal and 0 for every other transition. It is also partially observable, as the agent can only see at most in a 7x7 grid in front of it. The maximum episode length in this environment is 640 steps.

We used a tabular exploration reward of the form
\begin{equation} \label{tabular_reward}
    F_t = \frac{\alpha}{n(s)},
\end{equation}
where $n(s)$ is the number of times a state has been previously visited within an episode, and $\alpha$ is a coefficient controlling the magnitude of exploration reward relative to the environment reward. 
This particular tabular reward form appears in previous literature \cite{strehl2008analysis, burda2018exploration}.  As the environment itself is partially observable, and we wish to demonstrate the versatility of our method when applied to reward functions with dependence on arbitrary variables, we defined the ``state'' counter $n_t$ not based on the agent's observation space, but on information internal to the environment itself regarding the agent's position and whether it was holding the key. So, for example, the first time the agent visited state $\{3,4,0\}$ (\{``3rd vertical position,'' ``4th horizontal position,'' and ``not carrying the key''\}), it would receive a reward of $1$, followed by $\frac{1}{2}$, $\frac{1}{3}$, etc. To avoid a combinatorial explosion, the agent's direction and the door's status (locked/unlocked/open) were not incorporated into this reward.

Our experiment was in the Minigrid Doorkey 8x8 environment with a tabular exploration reward. 
We used the PPO algorithm as introduced by \citeay{schulman2017proximal}. We included an LSTM layer \cite{hochreiter1997long} in the network architecture to deal with the non-Markovian nature of the environment.
We tested four reward schemes: our method as implemented in Equation~\ref{conversion}, base intrinsic rewards without PBRS, our method as implemented in Equation~\ref{naive_conversion}, and control, which received no IM. We consider a policy to converge when it does not truncate an episode without reaching the goal state during the last 1,000 training steps.

We tested for three different sets of parameters $\alpha,\gamma$, meant to represent settings where the IM reward changes the optimal policy infrequently, sometimes, and often.\footnote{We include details of the effects of these parameters on the optimal policy in the Appendix.} Episode lengths for each of these are depicted in Figures~\ref{frame_results_005},~\ref{frame_results_02},~\&~ \ref{fig:025-full}, respectively. Shaded regions represent standard deviation among the 16 processes, rather than error: this is expected to be high, because MiniGrid environments are procedurally generated and variance in the optimal path length from one episode to another is expected. Tables~\ref{.005},~\ref{.02},~\&~\ref{.025} contain the time to convergence $T$ (if converged, N/A if they did not converge), and mean episode length $\bar{N}$ and standard deviation $\sigma$ after convergence for each reward scheme. $\bar{N}$ and $\sigma$ were calculated from the last 1,000 data points. $T$ was determined by the first time step in which the average episode length falls below $\bar{N}$ for that run. For each pair of results that converged we performed a 1-sided T-test, and
with one exception noted in Figure~\ref{frame_results}'s caption, all
differences were highly statistically significant.

\begin{figure*}[t]
    \centering
    \begin{subfigure}{0.24\textwidth}
        \centering
        \includegraphics[width=\textwidth]{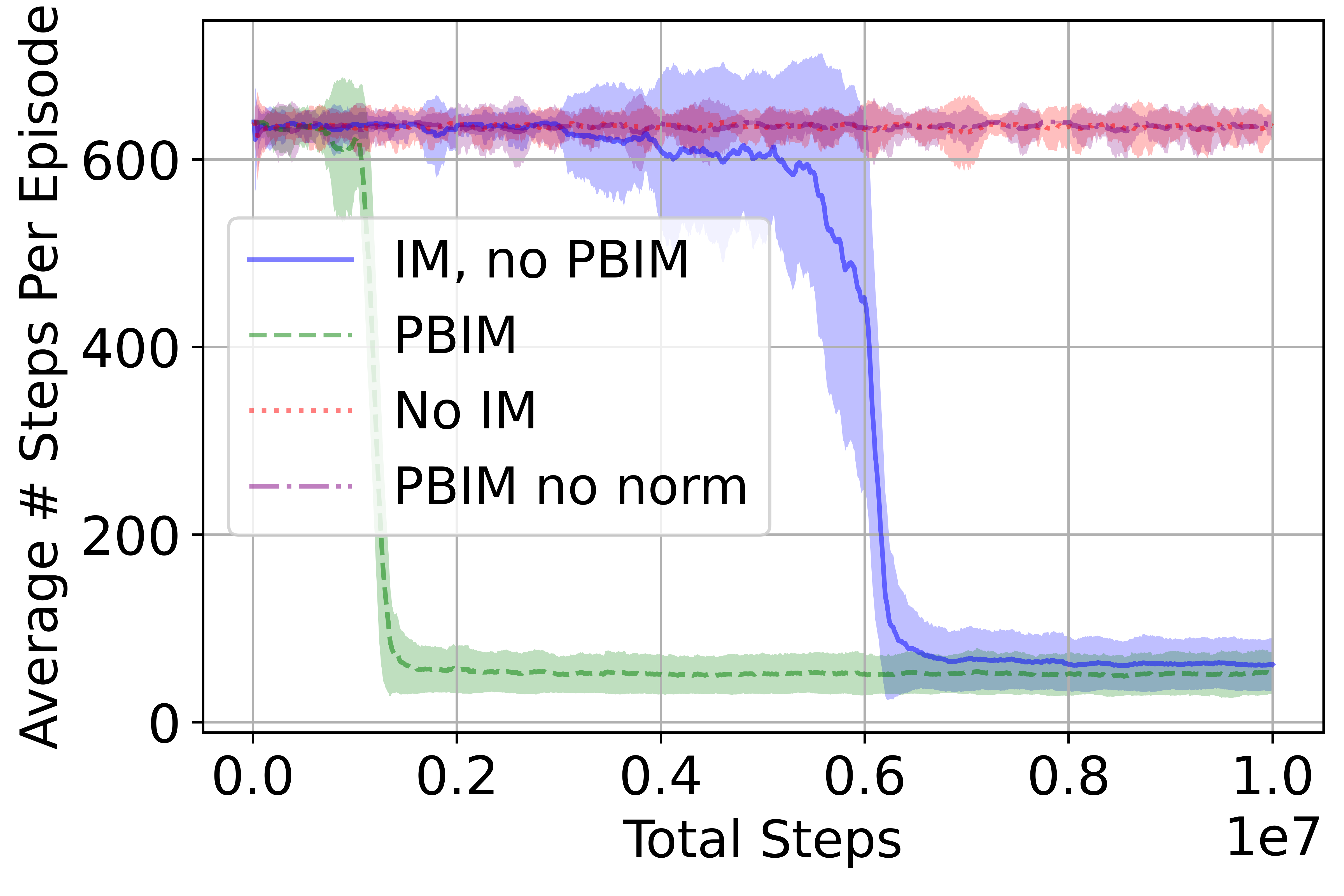}
        \caption{$\alpha = .025$, $\gamma = .995$}
        \label{fig:025-full}
    \end{subfigure}
    \begin{subfigure}{0.24\textwidth}
        \centering
        \includegraphics[width=\textwidth]{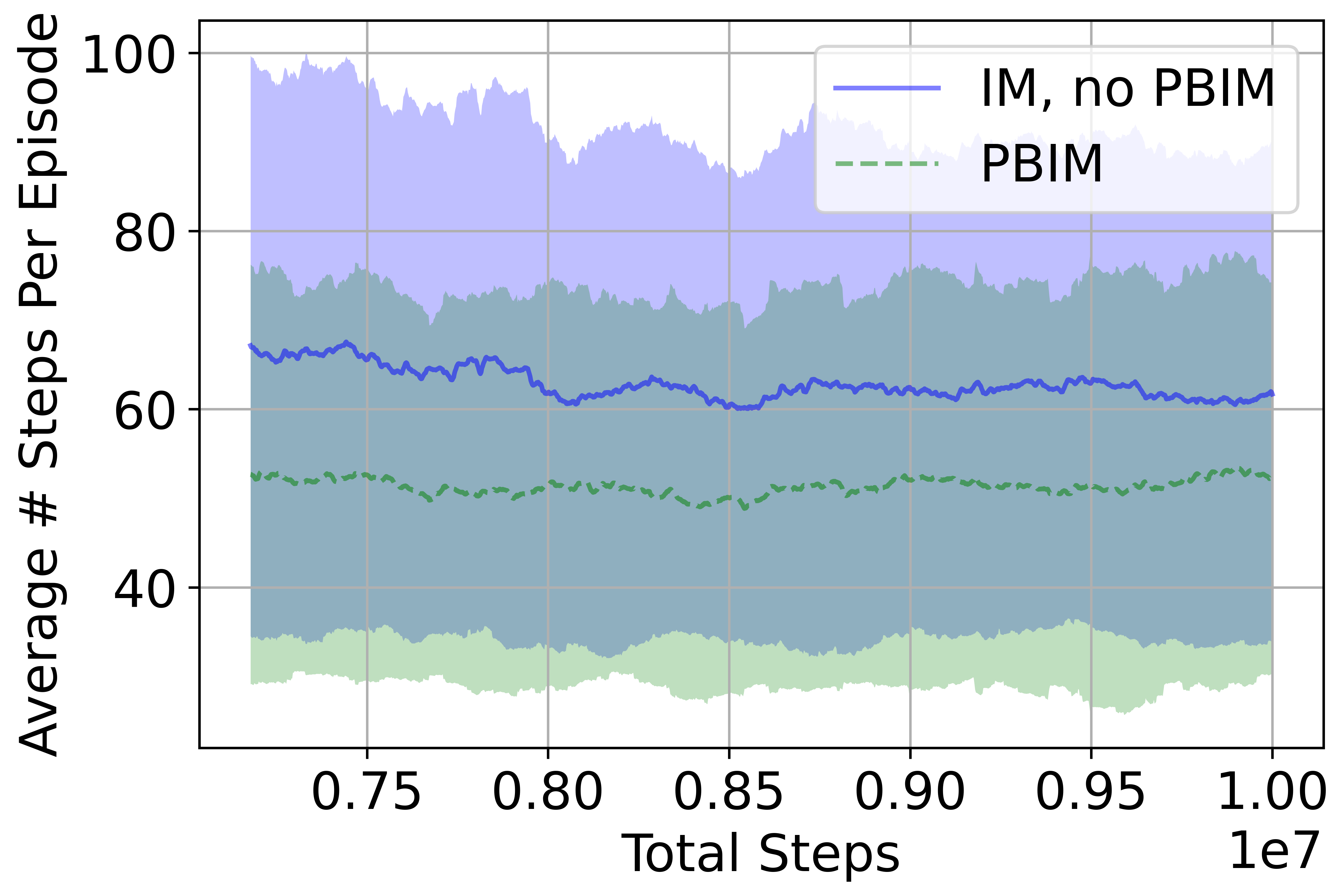}
        \caption{$\alpha = .025$, $\gamma = .995$, zoomed}
        \label{fig:zoomed}
    \end{subfigure}
    \begin{subfigure}{0.24\textwidth}
        \centering
        \includegraphics[width=\textwidth]{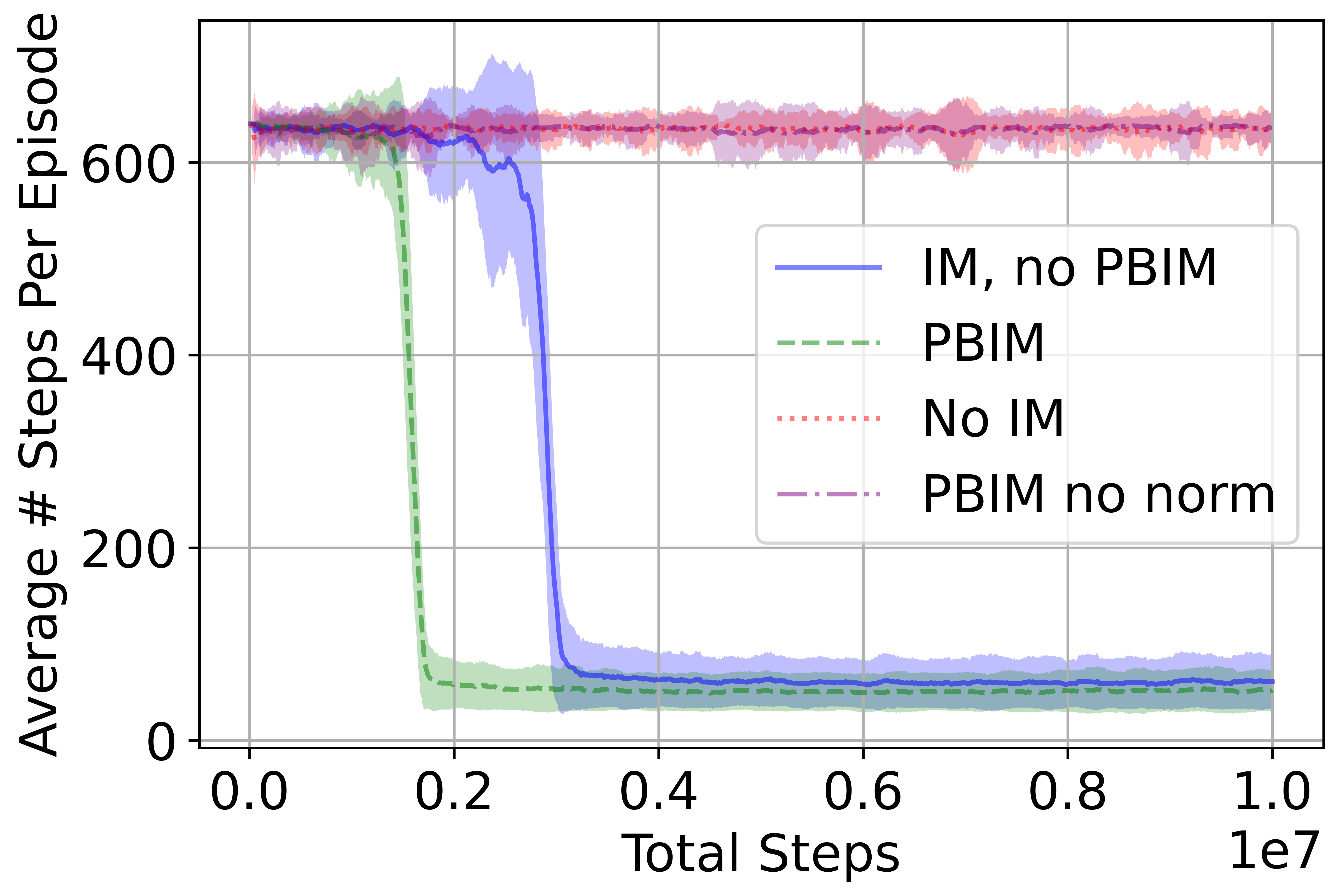}
        \caption{$\alpha = .02$, $\gamma = .995$}
        \label{frame_results_02}
    \end{subfigure}
    \begin{subfigure}{0.24\textwidth}
        \centering
        \includegraphics[width=\textwidth]{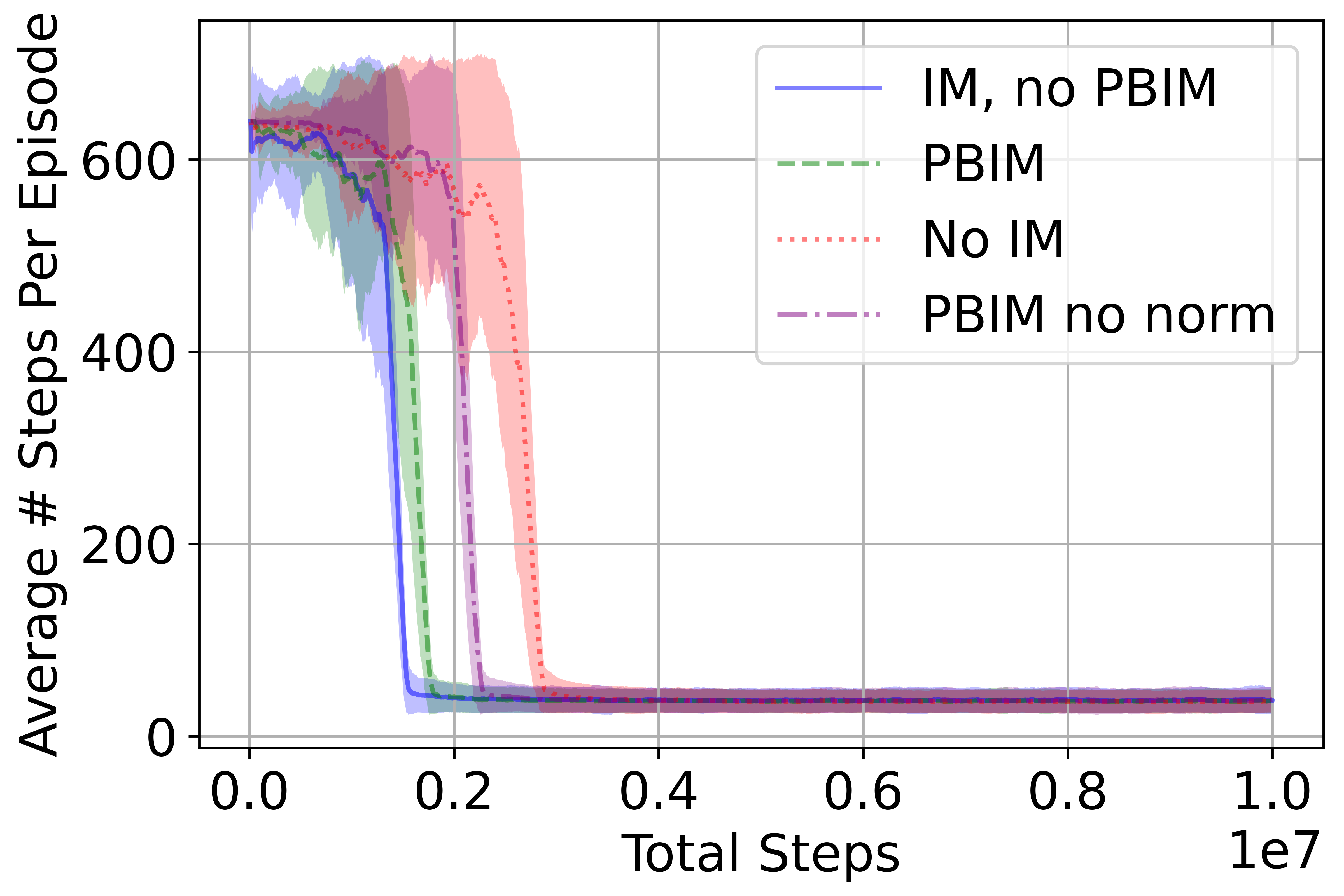}
        \caption{$\alpha = .005$, $\gamma = .99$}
        \label{frame_results_005}
    \end{subfigure}
    \caption{(\subref{fig:025-full}), (\subref{frame_results_02}), \& (\subref{frame_results_005}) Frames per episode for each method (lower is better). The shaded region represents standard deviation, and plots are of a 100-point moving average. \subref{fig:zoomed}) Same results as (\subref{fig:025-full}), but zoomed in. All differences in means are significant. For IM + PBRS, IM no PBRS (\subref{fig:025-full}) $T = 36.5, p < 0.01$. For IM + PBRS, IM no PBRS (\subref{frame_results_02}) $T = 27.4, p < 0.01$. In (\subref{frame_results_005}), No IM converges lower than IM + PBRS, which converges lower than IM + PBRS no norm, which converges lower than IM no PBRS. Respectively, for each of these pairings, $T = 4.3, p < 0.01$, $T = 6.1, p < 0.01$, and $T = 1.8, p = 0.32$. While the last of these isn't significant, the difference between IM + PBRS and IM no PBRS is, with $T = 7.9, p < 0.01$.}
    \label{frame_results}
\end{figure*}


%
\begin{table*}[t]
    \begin{center}
        \begin{small}
            \begin{sc}
                \begin{tabular}{lccc|ccc|ccc}
                    \toprule
                    & \multicolumn{3}{c|}{$\alpha=0.005$, $\gamma = 0.99$} & \multicolumn{3}{|c|}{$\alpha = 0.02$, $\gamma = 0.995$} & \multicolumn{3}{|c}{$\alpha = 0.025$, $\gamma = 0.995$} \\
                    \midrule
                    & $T$ & $\bar{N}$ & $\bar{\sigma}$ & $T$ & $\bar{N}$ & $\bar{\sigma}$ & $T$ & $\bar{N}$ & $\bar{\sigma}$ \\
                    \midrule
                    PBIM    & 1.67e6& 36.4&12.5 &           1.67e6&51.8 & 22.4 &            1.26e6&51.2& 22.7\\
                    IM, NO PBIM & 1.46e6&37.3&13.1&         2.88e6&60.5&27.5 &              6.07e6& 62.0&27.9\\
                    PBIM NO NORM   & 2.29e6& 37.1 & 13.4&   N/A & 635.9&14.9 &              N/A & 634.7 &18.8\\
                    NO IM   & 2.95e6&35.9&  12.0    &       N/A&634.8&  18.55    &         N/A&634.8&  18.6      \\
                    \bottomrule
                \end{tabular}
            \end{sc}
        \end{small}
    \end{center}
    \caption{Time to convergence ($T$), mean steps per episode after convergence ($\bar{N}$), and average standard deviation of steps per episode after convergence ($\bar{\sigma}$) for three parameter settings. Lower $\bar{N}$ is better.}
    \label{.005} \label{.02} \label{.025}\label{table_minigrid}
\end{table*}

\subsubsection{Discussion}

As can be seen in Figure~\ref{frame_results} and Table~\ref{table_minigrid}, our method consistently outperformed the baseline IM method. Additionally, when the IM most often changes the optimal policy (Figure~\ref{fig:025-full}), our method converges faster than the baseline, and as predicted, the degree to which it outperforms the baseline is correlated with the frequency with which the optimal policy is altered. Note also that our modification of Equation~\ref{naive_conversion} into Equation~\ref{conversion} was key in allowing for convergence in the more difficult environments of Figures~\ref{frame_results_02}~\&~\ref{fig:025-full}, and in outperforming the IM baseline to a statistically significant degree in Figure~\ref{frame_results_005}. We also obtain consistently lower variance than the IM baseline in all experiments: as the level of variance in episode length due to differences in initial conditions is unchanged between runs, this suggests that there is quite a bit of additional variance from inconsistent performance in the baseline IM method that is not present in our method. 

The only experiment in which our (normalized) method did not perform best in both speed of convergence and final policy was with $\alpha = 0.005, \gamma = 0.99$. Here, as can be seen in Table~\ref{table_minigrid}, the best-performing policy after convergence was that trained with no IM at all, and the fastest-converging was that trained with the baseline IM. The latter of these observations can be explained by noting, as discussed in Section~\ref{converting}, that there is a reward horizon for the lack of intrinsic rewards' utility in the PBRS agent, and in simpler environments, there is a risk of this reward horizon being successfully learned before the environment itself is fully solved. If this happens, PBRS can slow the speed of convergence, rather than increase it, by teaching the agent to ignore the IM term prematurely. Note, though, that our method converged to a policy more efficient than that of IM to a significant degree,\footnote{Another point of note is that the converged values in Figure~\ref{frame_results_005} are lower than that of Figures \ref{frame_results_02} and \ref{fig:025-full}. We attribute this to the higher $\gamma$ value making the environment more difficult to learn, meaning that the policies are likely less closely optimal.} and converged more quickly than the run with no IM. In this worst-case scenario our method still provides value by facilitating a trade-off between preventing reward hacking and increasing training efficiency.

\subsection{Cliff Walking}
\label{sec:cw}

We also tested our method in a cliff walking~\cite{sutton2018reinforcement}, a classic reinforcement learning task in which an agent is directed to find a goal state at the end of a long ``cliff'' that must be avoided. Details of this environment are described in the Appendix. 
For this experiment, we trained a simple, tabular Q-learning agent with four types of intrinsic motivation: none, RND\cite{burda2018exploration}, non-normalized PBIM, following Equation \ref{naive_conversion}, and normalized PBIM, following Equation \ref{conversion}. We used an RND predictor network in this environment, to test our method's ability to accommodate complex, non-tabular, state-of-the-art IM terms with dependence on full training trajectories.  More details of this environment can be found in the Appendix.

\subsubsection{Discussion}
Our results in this environment are plotted in Figure \ref{small_cliffwalker}. Episode length is a rough indicator of exploration in the environment. We expect agents to start with very short episodes as they discover falling off the cliff is not ideal, then increase the episode length as they explore the environment, while finally going down as they find the optimal route.

RND, in this environment, doesn't reach the goal, as it becomes ``distracted'' by the intrinsic rewards that can be obtained by exploring the environment unnecessarily---similar to the optimal-policy-changing effect of the tabular exploration term in Section \ref{minigrid}.\footnote{This is not an indictment of RND itself; we shouldn't expect it to be helpful in this small of an environment. We implemented it here partially to see if PBIM could mitigate an intentionally ``bad'' (for this environment) IM term.} Similarly, unnormalized PBIM doesn't reach the goal, likely due to biasing issues discussed in Section \ref{converting}. Normalized PBIM, however, successfully reaches the same average external return as the baseline no-IM runs, suggesting it successfully staves off this change to the optimal policy.

\begin{figure}
    \centering
    \begin{subfigure}{0.49\columnwidth}
        \includegraphics[width=\columnwidth]{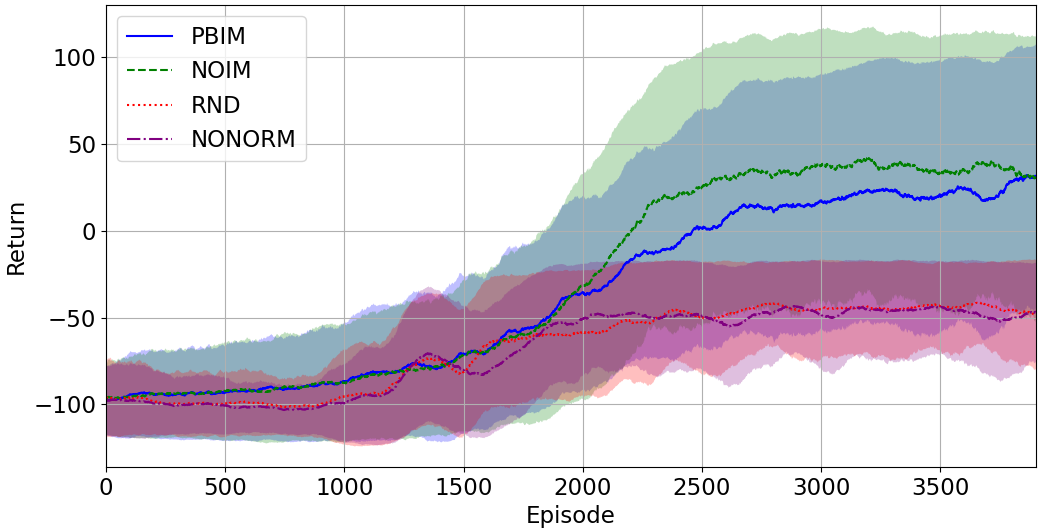}
        \caption{Average Return}
        \label{fig:returnsmall}
    \end{subfigure}
    \begin{subfigure}{0.49\columnwidth}
        \includegraphics[width=\columnwidth]{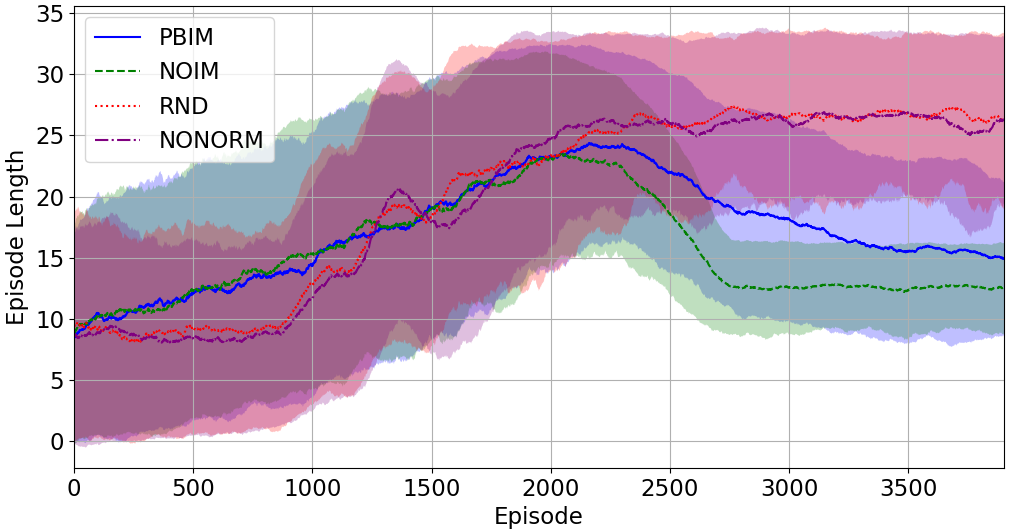}
        \caption{Episode Length}
        \label{fig:lengthsmall}
    \end{subfigure}
    \caption{Average cumulative extrinsic return and episode length for the cliff walking environment. Error bars are standard deviations over 10 runs. Differences in means between returns of No IM and RND ($p < 0.05$) and between returns of No IM and PBIM No Norm ($p < 0.05$) are statistically significant. Mean episode lengths of both PBIM norm and no IM are statistically different from both PBIM no norm and RND ($p < .05$ for all).}
    \label{small_cliffwalker}
\end{figure}
Furthermore, in Figure \ref{Q_values}, it is apparent that normalized PBIM shares a ``critical path,'' one that follows the shortest possible route along the cliffside from the start state to the goal state. This translates, effectively, to equivalent policies, particularly in the deterministic version of this environment that we are using.\footnote{Note that, with the exception of two outliers in infrequently-visited parts of the state space, all differences between \ref{fig:noim} and \ref{fig:pbimnorm} involve the PBIM implementation taking an ``equally correct'' path: in particular, traveling rightwards instead of downwards.} In stark contrast, the other policies fail to converge to any meaningful goal-seeking and have policies that differ starkly from the no-IM version.

\begin{figure}[t]
    \begin{subfigure}{0.49\columnwidth}
        \centering
        \includegraphics[width=\columnwidth]{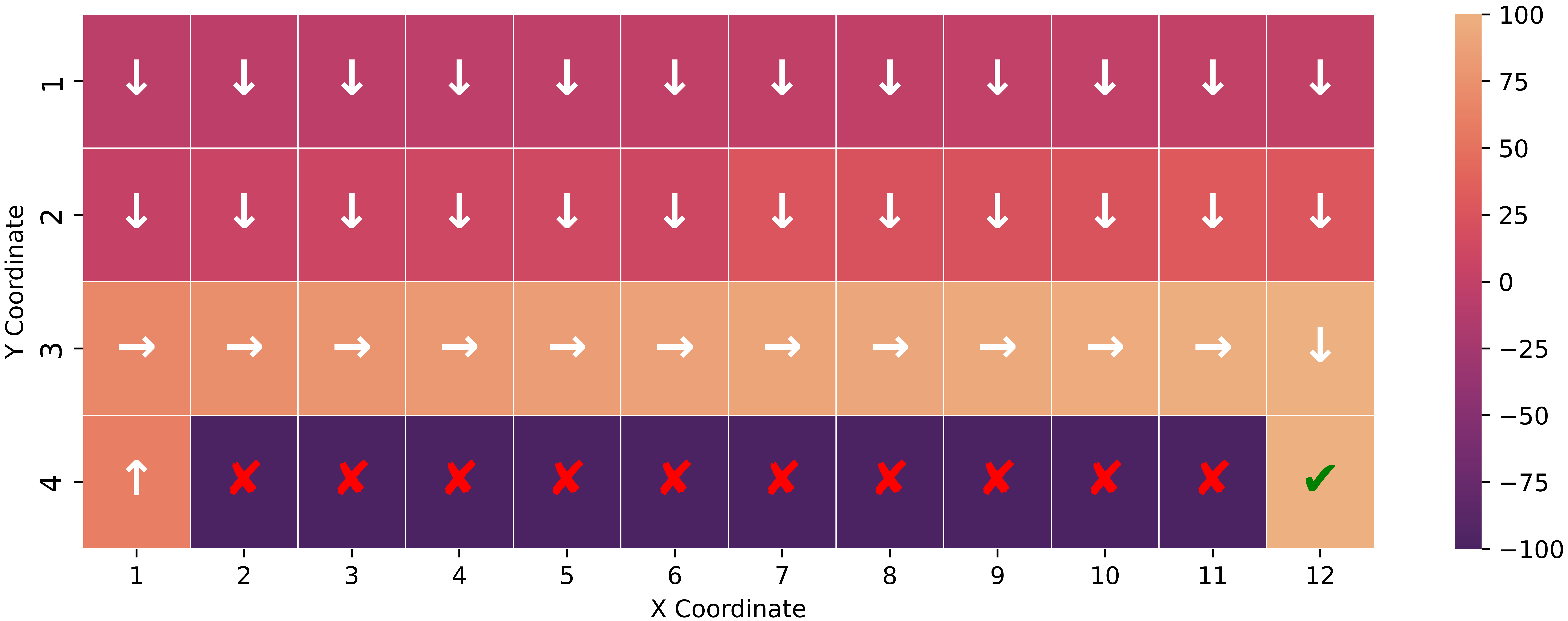}
        \caption{No IM}
        \label{fig:noim}
    \end{subfigure}
    \begin{subfigure}{0.49\columnwidth}
        \centering
        \includegraphics[width=\columnwidth]{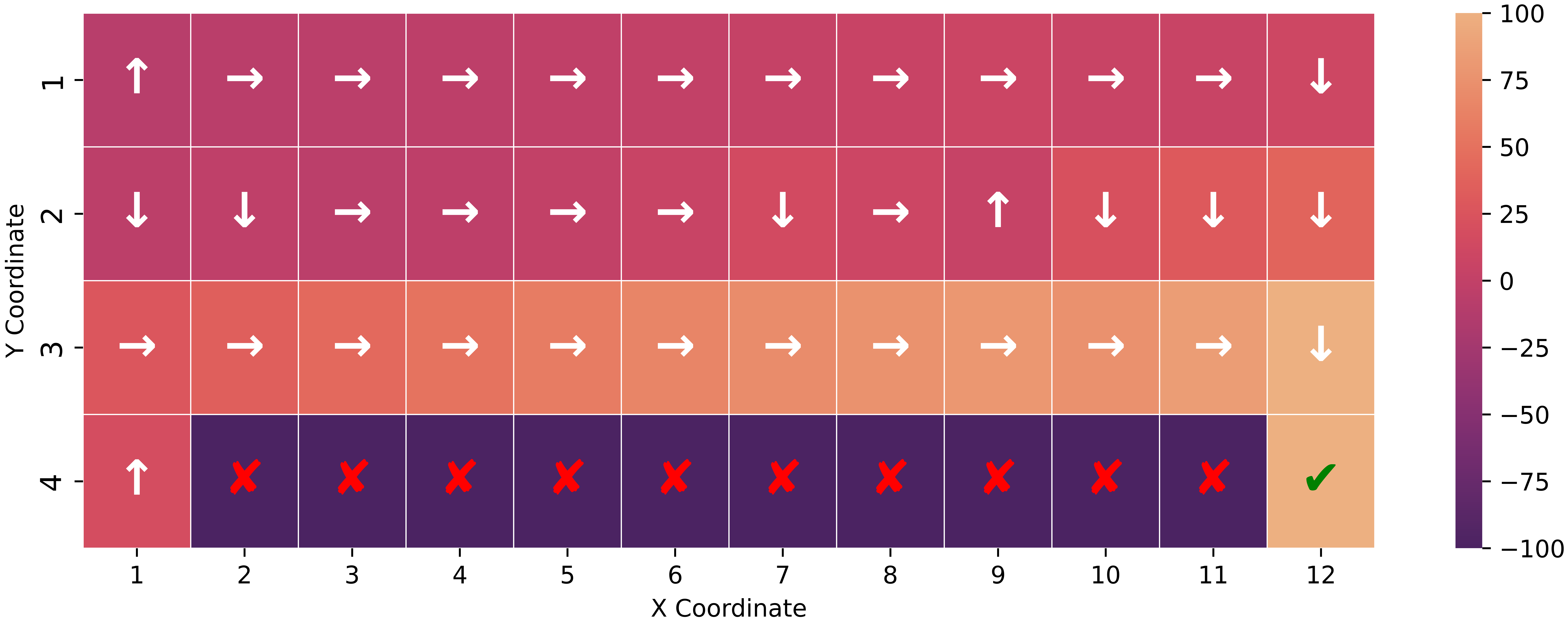}
        \caption{PBIM, normalized}
        \label{fig:pbimnorm}
    \end{subfigure}
    \begin{subfigure}{0.49\columnwidth}
        \centering
        \includegraphics[width=\columnwidth]{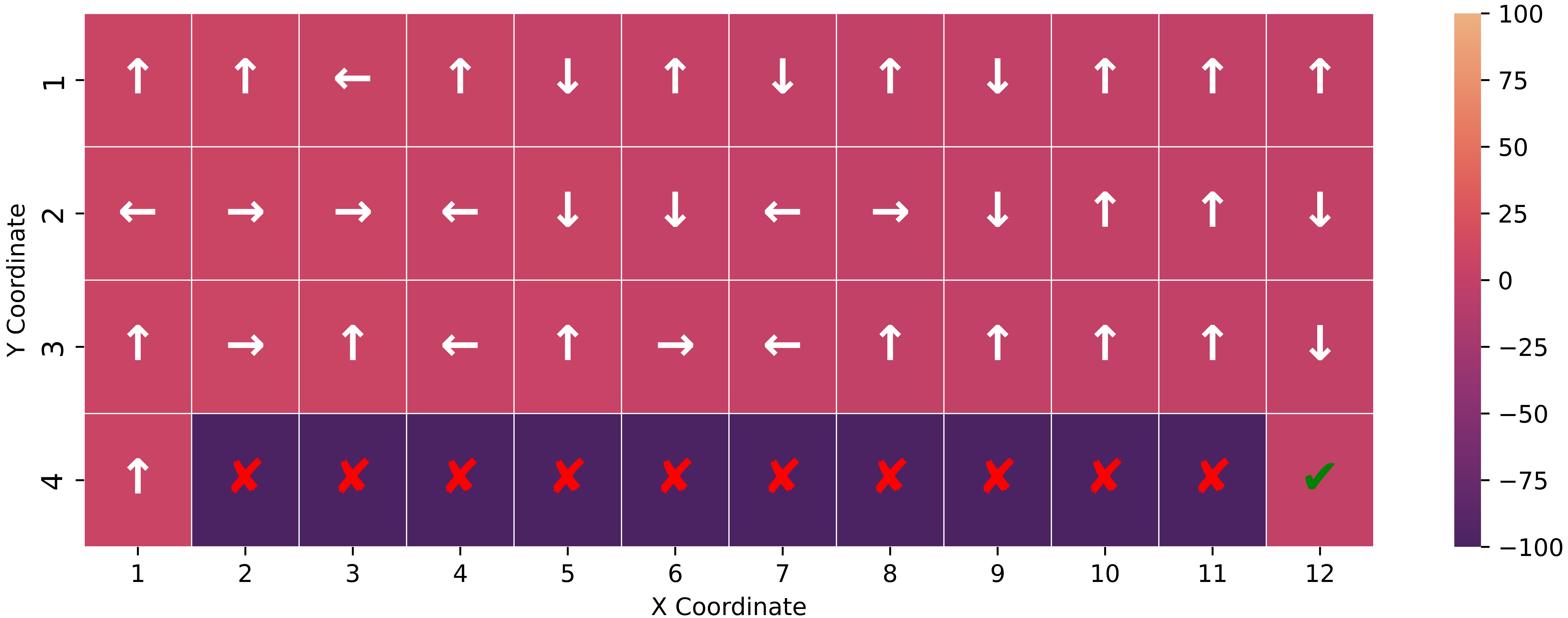}
        \caption{PBIM, not normalized}
        \label{fig:pbimnonorm}
    \end{subfigure}
    \begin{subfigure}{0.49\columnwidth}
        \centering
        \includegraphics[width=\columnwidth]{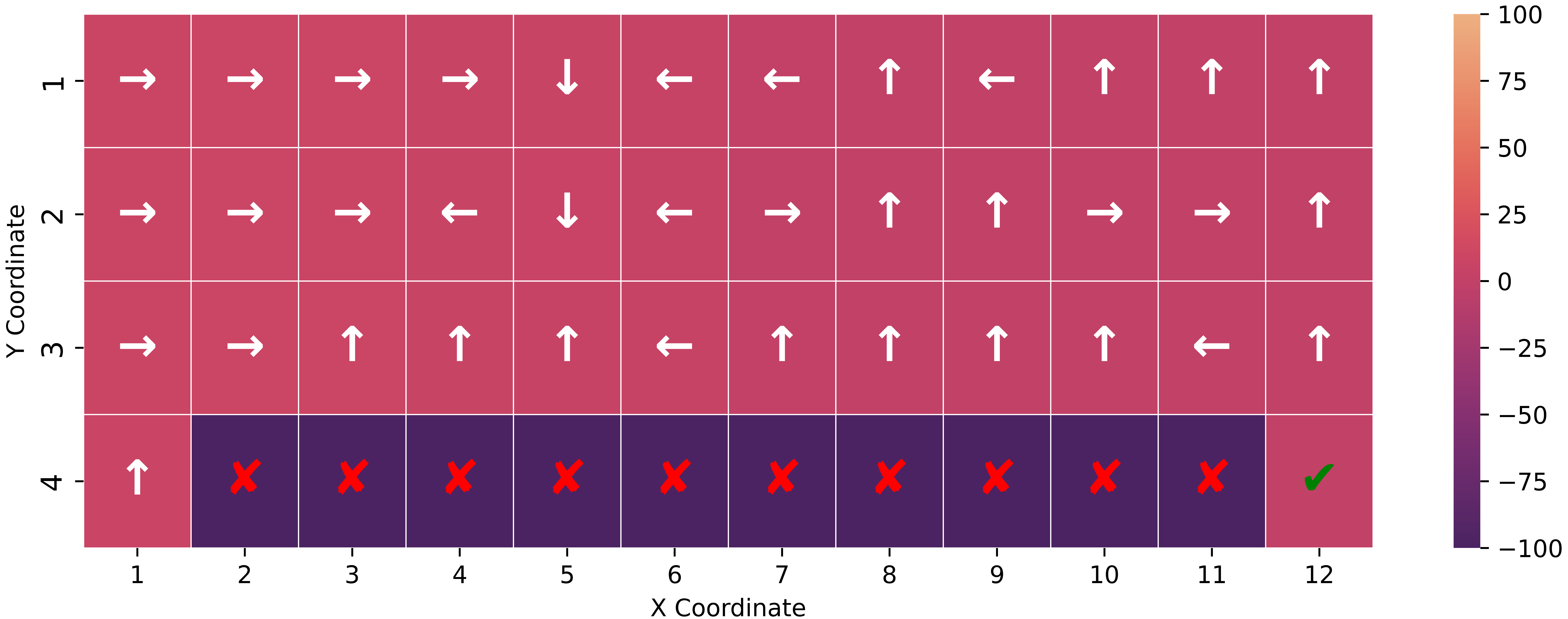}
        \caption{RND}
        \label{fig:rnd}
    \end{subfigure}
    \caption{Final policies of trained agents and their estimated Q-values. Arrows indicate the action with the highest estimated Q-value in each position. A brighter hue indicates a higher Q-value.}
    \label{Q_values}
\end{figure}

\subsection{Longer Cliff Walking}
The environment in Section \ref{sec:cw} was simple enough that IM was not necessary to find a solution. In order to test our method's efficacy in more dauntingly sparse-reward environments that require reward shaping to solve effectively, we evaluate PBIM in a modified version of Cliff Walking with a much longer grid. This version features a $4 \times 50$ grid, where the start and goal are on the leftmost and rightmost tiles in the bottom row, and all other bottom tiles are cliffs.

Details of the changed parameters from those in Section \ref{sec:cw} are in the Appendix: most notably, we increased the number of episodes to train for, as well as the maximum episode length. Figure~\ref{fig:large_cliffwalker} shows the return for the three Q-learning agents trained in this experiment: no IM, RND, and normalized PBIM\footnote{We did not train unnormalized PBIM in this environment due to its poor performance in the similar, simpler environment in Section \ref{sec:cw}.}. Results are aggregated over 10 runs. The agent policies, along with further details, are included in the Appendix.

\begin{figure}
    \centering
        \includegraphics[width=\linewidth]{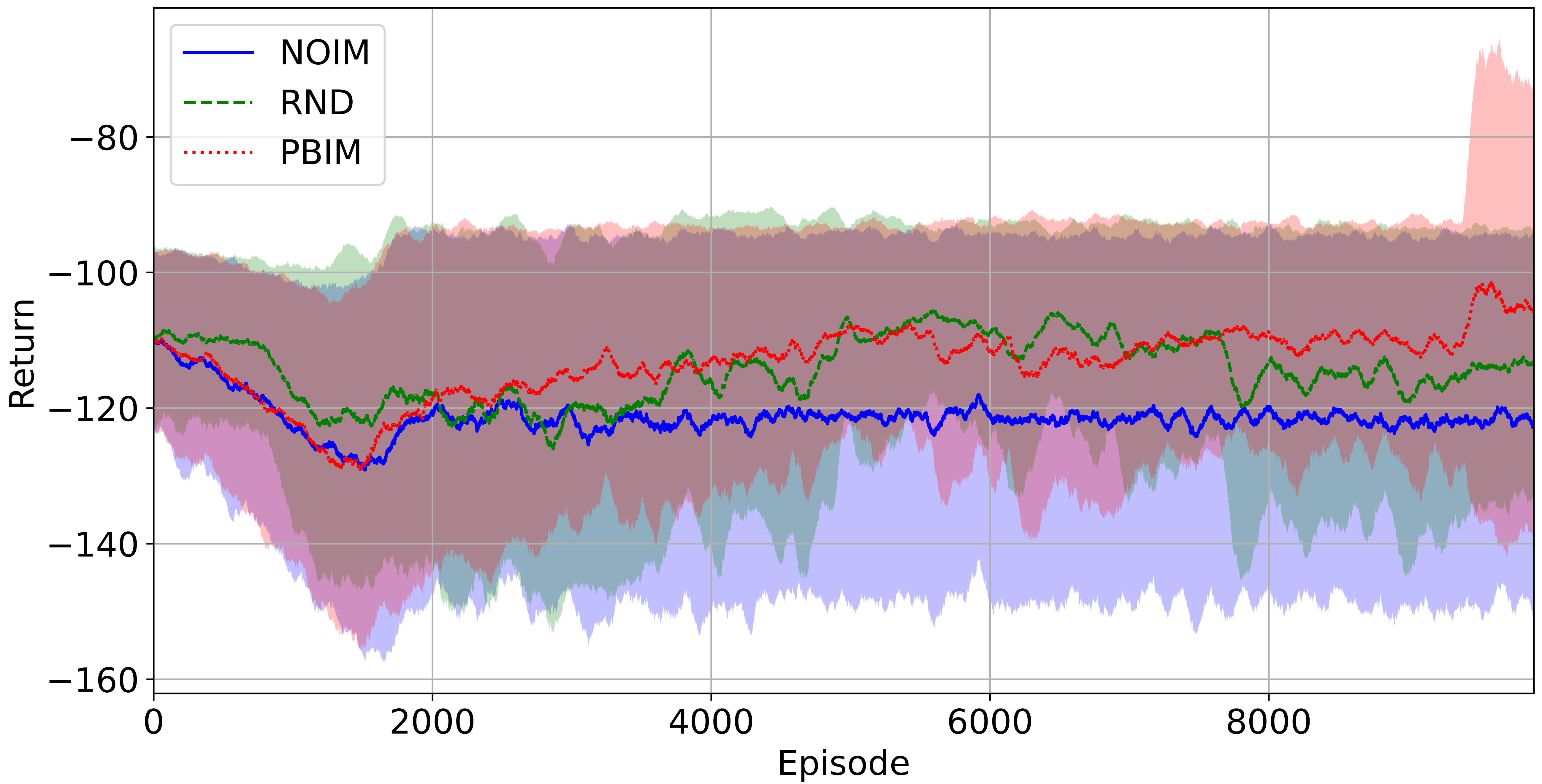}
    \caption{Average cumulative return for the large cliff walking environment. }
    \label{fig:large_cliffwalker}
\end{figure}

The Q-learning agents with IM rewards (RND and PBIM) obtained more average returns 
per episode after 3,000 iterations. This difference above no-IM is statistically significant for PBIM ($p<0.05$), but not for RND, suggesting this is another environment in which PBIM can speed training more efficiently than RND alone.

Note however that, while PBIM is statistically superior to the no IM run, none of the agents, within the time allotted, are able to converge to a consistently stable policy. This can be observed in Figure~\ref{fig:large_cliffwalker}, as the reward values constantly hover below -100, indicating agents on average explore the environment and then fall into the cliff. The PBIM agent, however, began to break out of this cycle near the end of training, while the other two agents remained stuck. This preliminary experiment shows promise for the application of PBIM to more complex sparse-reward environments, as a method of speeding up training over traditional IM. 
However, it also suggests the need for future work (both in this environment and other sparse reward environments), to better refine and characterize the bounds of and conditions for this claim. 

\section{Conclusion}

We've extended PBRS to a more general class of reward functions than has been covered previously in the literature, and proven that important theoretical guarantees---namely, the preservation of the set of optimal policies for the underlying environment---still hold. We have also provided a computationally efficient and effective method of converting many state-of-the-art IM methods into this optimality-preserving form and demonstrated its efficacy at both preventing IM reward hacking and, in some circumstances, accelerating training.

In future work, we are interested in investigating how other forms of IM can be combined with PBIM to positively influence an agent, particularly in more complex environments.

\bibliographystyle{ACM-Reference-Format} 
\bibliography{biblio}
\appendix
\appendix

\section{MiniGrid DoorKey Environment}

An example of a state from the environment we used is given in Figure \ref{doorkey}.

For the potential-based method, we implemented Equations 30 and 34, taking $F_n$ from Equation 41. We used a moving average of the rewards for each run to compute $\bar{F}$.

We ran 16 processes concurrently for 10 million cumulative steps, and recorded performance every 128 steps. 

It is worth expanding on the some of particular ways that the IM in this experiment changed the set of optimal policies in this environment. With Equation 41, the particular intrinsic reward we used in this environment---though certainly not unique to it; see \cite{burda2018large} as well as Section 4.2---there is a danger of the agent ``procrastinating,'' and collecting more exploration rewards than are strictly necessary to reach the goal state, rather than taking the most efficient action at each time step. Let's examine a simple version of this: if an agent can either take the most efficient action to proceed toward the goal state, or wait for $t$ time steps before doing so (say, by either alternating ``turn left'' and ``turn right'' actions, or by taking the ``toggle'' action repeatedly when there is no door to toggle), then which course of action is preferred to the other depends on its intrinsic reward for the current grid tile. If it has visited its current tile a total of $n$ times (including the current time step), then stalling for $t$ time steps before proceeding to the goal becomes strictly preferred to the optimal action if the condition
\begin{equation} \label{procrastination_condition}
    1-\gamma^t < \sum_{t' = 0}^{t-1}\alpha \gamma^{t'}\frac{1}{n+t'+ 1}
\end{equation}
is met. Here, the left side of Equation \ref{procrastination_condition} refers to the cost of delaying the step at which it will reach the goal state by $t$ time steps, and the right side refers to the intrinsic reward gained by doing the same. 

Thus the set of optimal policies in this environment can be altered by the shaping reward in Equation 41 for sufficiently high values of $\alpha$ and $\gamma$. To test how violations of this inequality affect performance, we tested at values of $\alpha$ and $\gamma$ which vary in the frequency with which we can expect Equation~\ref{procrastination_condition} to be satisfied. In particular, higher values of both $\alpha$ (a positive coefficient in the right side of Equation \ref{procrastination_condition}) and $\gamma$ (positive on the right side of Equation \ref{procrastination_condition} and negative on the right side) will increase the frequency of states in which the optimal policy will be altered by this intrinsic reward.

We iterated on the initial codebase of \cite{minigrid_code} for our implementation.

\begin{figure}
    \centering
    \includegraphics[width=0.8\columnwidth]{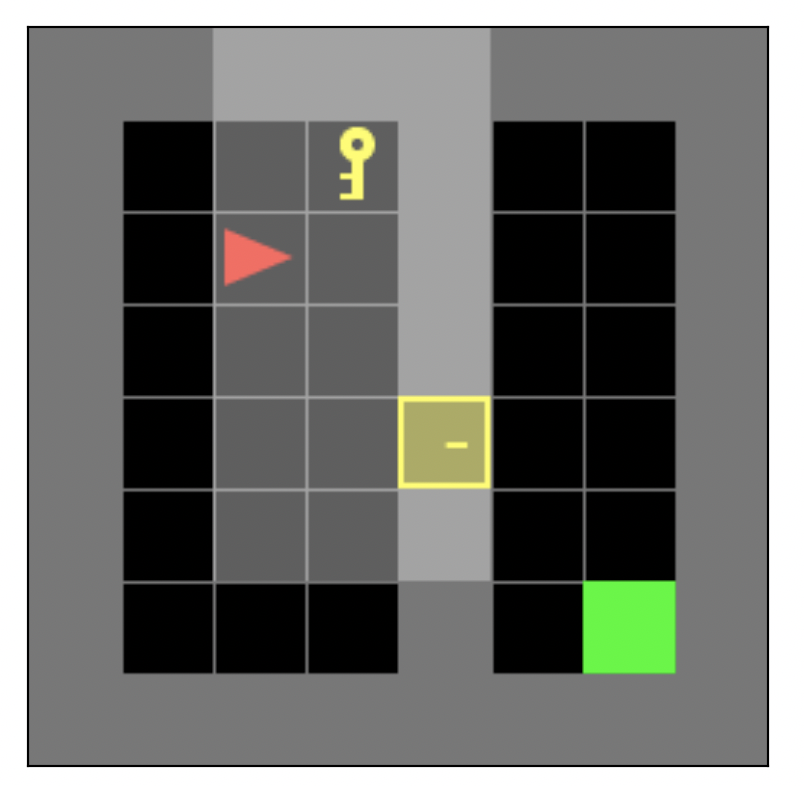}
    \caption{An example MiniGrid DoorKey 8x8 environment.}
    \label{doorkey}
\end{figure} 
\section{Cliff Walking Environment}  
Here, we detail the specifics of the cliffwalking experiment.
\subsection{Environmental Setup}
Figure~\ref{fig:cw} shows the grid and mechanics of the cliff walking task. 
In this environment, the agent starts in the bottom left tile, marked with an S, and must reach the bottom right tile, marked with a G. All other tiles in the bottom row are ``cliffs.'' At every time step, the agent can move either up, down, left, or right. Entering cliff tiles returns a reward of -100, and reaching the goal tile instead returns 100. Every other action has a reward of -1. The episode ends when either the agent reaches a cliff or goal tile, or when the maximum episode length is reached. 
In Figure~\ref{fig:cw}, the red arrow indicates the optimal path\footnote{We use a deterministic implementation of the environment, in which the ``slip chance'' is set to 0. In nondeterministic versions of the environment, depending on the value of this parameter and the penalty for moving into a cliff tile, the blue path may actually be optimal. We use the deterministic implementation partially because it's easier to demonstrate and think about optimality here.}, and the blue arrow shows the longer path that an agent might take to ensure it avoids the cliffs.

\begin{figure}
    \centering
    \includegraphics[width=\linewidth]{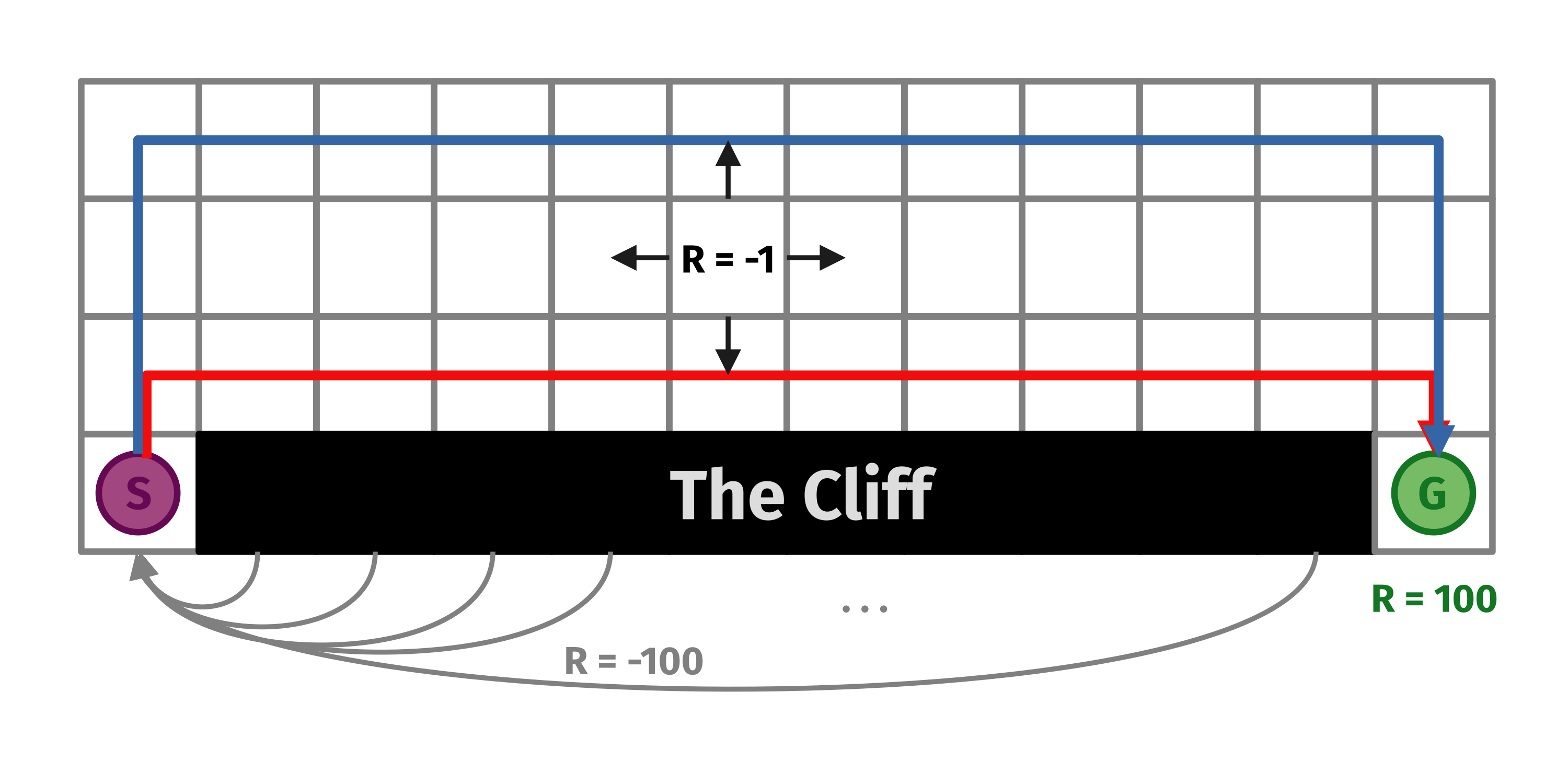}
    \caption{Cliff walking scenario (adapted from~\cite{sutton2018reinforcement}).}
    \label{fig:cw}
\end{figure}

\subsection{Hyperparameters}
The q-learning agents were trained for 4000 episodes with $\gamma = 0.99$, $\epsilon = 1.0$ decreasing by $4 \cdot 10^{-4}$ after every episode down to a minimum of $0.1$. Due to the small observation space, we used an RND predictor network learning rate of $10^{-6}$ and an intrinsic reward factor\footnote{RND returns an intrinsic reward based on the agent's ability to navigate to new states that return a high prediction error for a separately trained network. Because this was a simple deterministic environment, there was a danger of the predictor network quickly converging to zero prediction error, effectively eliminating the difference between the models we were testing. We set these parameters to counteract this effect, and maintain a consistently nonzero IM, in order to adequately test our method.} of $1000$.

\section{Long Cliff Walking}
We used the same experimental setup as the Cliff Walking Environment, but instead increased the number of episodes to 10,000 and the maximum steps per episode to 100. The agents were trained with $\gamma = 0.99$, $\epsilon = 1.0$ decreasing by $5 \cdot 10^{-4}$ after every episode down to a minimum of $0.1$. We used the same RND configuration: a learning rate of $10^{-6}$ and reward scaling of $1000$.

Figure~\ref{fig:large_cliffwalking_policies} shows the policies for the three trained agents, averaged over 10 runs. None of the three policies converge to a stable path, which was the case for the shorter cliff walking environment. PBIM and no IM's policy partially exhibit the behavior we want the agent to follow: move towards the right side of the grid and not fall over the cliff. In the case of RND, exploration was insufficient, as it was unable to reach the states beyond the 42nd column. The Q-values are relatively low on the right side, as the agents have yet to find the goal reward. In the case of PBIM, it has found the reward but has yet to propagate it back.

\begin{figure*}[t]
    \begin{subfigure}{\textwidth}
        \centering
        \includegraphics[width=\textwidth]{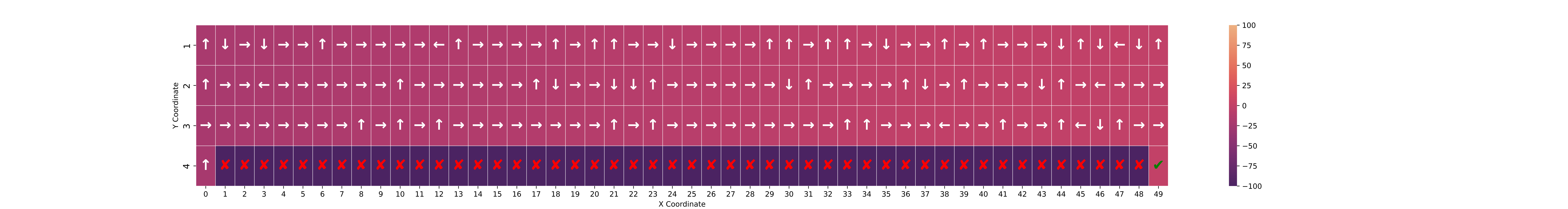}
        \caption{No IM}
        \label{fig:large-noim}
    \end{subfigure}
    \begin{subfigure}{\textwidth}
        \centering
        \includegraphics[width=\textwidth]{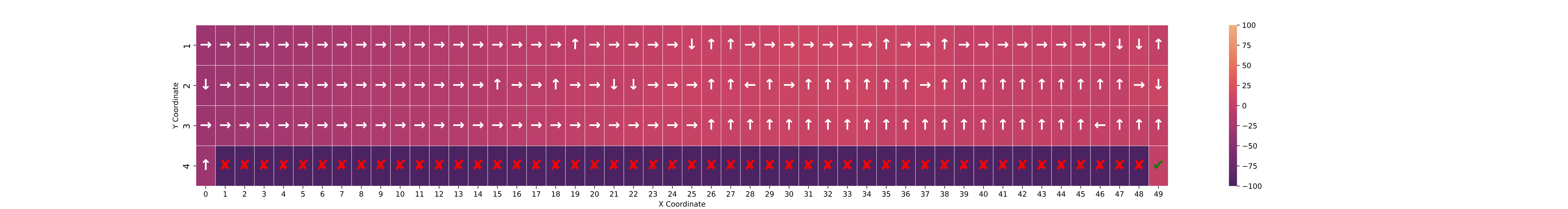}
        \caption{PBIM, normalized}
        \label{fig:large-pbimnorm}
    \end{subfigure}
    \begin{subfigure}{\textwidth}
        \centering
        \includegraphics[width=\textwidth]{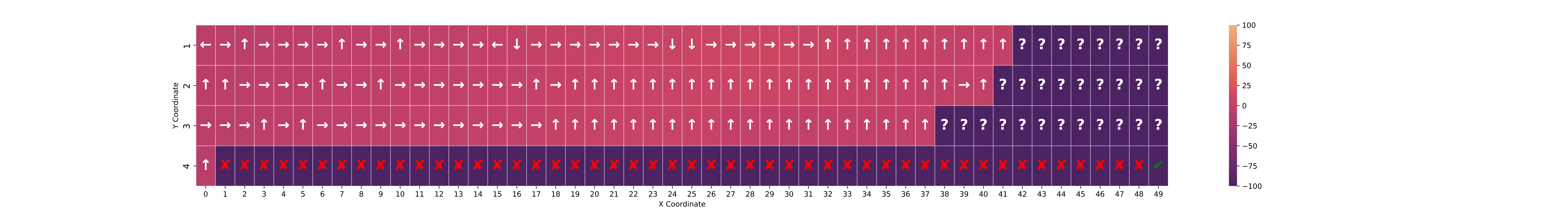}
        \caption{RND}
        \label{fig:large-rnd}
    \end{subfigure}
    \caption{Final policies of trained agents. Arrows point to the agent's movement in every given position. A brighter hue indicates a higher q-value. The darkened tiles marked with a `?' indicate that the agent did not reach the position.}
    \label{fig:large_cliffwalking_policies}
\end{figure*}


\end{document}